\documentclass{article}
\usepackage{fullpage}
\usepackage{wrapfig}
\usepackage{amsmath, amssymb, amsthm, graphicx}
\usepackage{algorithm, algorithmic}

\title{No Internal Regret via Neighborhood Watch}
\author{Dean Foster \\ Department of Statistics \\ University of Pennsylvania \and Alexander Rakhlin\\ Department of Statistics \\ University of Pennsylvania}

\newcommand{\A}{{\mathcal{A}}}
\newcommand{\E}{{\mathbb{E}}}
\newcommand{\C}{{\mathcal{C}}}
\newcommand{\F}{{\mathcal{F}}}
\newcommand{\G}{{\mathcal{G}}}
\renewcommand{\P}{{\mathbb{P}}}
\newcommand{\reals}{\ensuremath{\mathbb R}}
\newcommand{\tr}{\ensuremath{{\scriptscriptstyle\mathsf{T}}}}
\newcommand{\ind}[1]{{\bf I}\left\{#1\right\}}
\newcommand{\bone}{{\mathrm 1}}
\def\deq{\triangleq}
\newcommand{\algoname}{{\sf Neighborhood Watch}}

\newtheorem{theorem}{Theorem}[section]
\newtheorem{definition}{Definition}[section]

\newtheorem{lemma}{Lemma}[section]
\newtheorem{corollary}{Corollary}[section]

\begin{document}

\maketitle

\begin{abstract}
	We present an algorithm which attains $O(\sqrt{T})$ internal (and thus external) regret for finite games with partial monitoring under the \emph{local observability condition}. Recently, this condition has been shown by Bart\'ok, P\'al, and Szepesv\'ari \cite{BarPalSze11} to imply the $O(\sqrt{T})$ rate for partial monitoring games against an i.i.d. opponent, and the authors conjectured that the same holds for non-stochastic adversaries. Our result is in the affirmative, and it completes the characterization of possible rates for finite partial-monitoring games, an open question stated by Cesa-Bianchi, Lugosi, and Stoltz \cite{CesLugSto06}. Our regret guarantees also hold for the more general model of partial monitoring with random signals.
\end{abstract}

\section{Introduction}

Imagine playing a repeated zero-sum game against an opponent (column player) where the loss is defined by a given matrix $L\in \reals^{N\times M}$. Unlike the classical full-information scenario, however, we (the row player) do not observe the moves of the opponent and instead receive some signal given by the known matrix $H\in \Sigma^{N\times M}$ defined over some alphabet $\Sigma$. Specifically, for the choices $i$ and $j$ of the row and column players, the row player observes the signal $H_{i,j}$. Neither the move of the opponent nor the incurred loss $L_{i,j}$ is observed by the row player. In this paper, we are concerned with rates for external and internal regret achievable in this scenario.

The question of characterizing such rates in terms of the matrices $L$ and $H$ has been raised by Cesa-Bianchi, Lugosi, and Stoltz \cite{CesLugSto06}. Under a linear dependence between the matrices $L$ and $H$, the authors proved $O(T^{2/3})$ rates for external regret, yet noted that there exist games with the $\Theta(\sqrt{T})$ behavior (e.g. the so-called \emph{bandit feedback} games where $L=H$). Similar distinction in available rates also appears to hold for internal regret: an $O(T^{2/3})$ upper bound was shown in \cite{CesLugSto06}, while the rate of  $O(\sqrt{T})$ is achievable for bandit feedback by the result of Blum and Mansour \cite{BluMan07}. 

Recently, Bart\'ok, P\'al, and Szepesv\'ari in \cite{BarPalSze10,BarPalSze11} made key insights into the problem of partial monitoring. In particular, \cite{BarPalSze11} characterized the rates for external regret against an i.i.d. (stochastic) opponent. The authors showed that rates can only be one of $\Theta(1), \Theta(\sqrt{T}), \Theta(T^{2/3})$ and $\Theta(T)$, and that a so-called {\em local observability condition} plays a key role in determining this growth behavior. In the non-stochastic (adversarial) case, however, no general characterization is available to date, with the notable exception of games with two adversarial actions \cite{BarPalSze10}. As suggested by \cite{BarPalSze11}, to provide a complete characterization for external regret against non-stochastic opponents, it would be enough to show an upper bound of $O(\sqrt{T})$ under the \emph{local observability condition}. The characterization would follow because \cite{BarPalSze11} proves a $\Omega(T^{2/3})$ lower bound when local observability does not hold (yet the game is not hopeless with $\Omega(T)$ regret) and the upper bound of $O(T^{2/3})$ is achieved by the algorithm of Piccolboni and Schindelhauer \cite{PicSch01} through the analysis of \cite{CesLugSto06}. 

This paper presents an algorithm, \algoname, with an upper bound of $O(\sqrt{T})$ for both internal and external regret against a non-stochastic opponent under the local observability condition. Together with the results mentioned above, this completes the characterization for both internal and external regret. It is remarkable that the condition of local observability that characterizes games against a stochastic environment also characterizes games against non-stochastic opponents.

We now summarize our approach. First, we define a notion of \emph{local} internal regret which postulates that the player does not benefit by switching any of its actions to a neighboring action. The neighbor relation is defined by the neighborhood graph of best responses to mixed strategies of the opponent. Second, we show that small \emph{local} internal regret implies small (global) internal regret. We then present an algorithm which randomly chooses a neighborhood and then chooses an action in the neighborhood. A key property satisfied by the two-level procedure is a certain flow condition. Under this condition, external regret of sub-algorithms on local neighborhoods can be turned into a statement about local internal regret (and, hence, global internal regret). External regret of the sub-algorithms, in turn, can be upper bounded because local observability condition allows us to estimate relative losses of neighboring actions.

\section{Notation and definitions}

We follow the notation of \cite{BarPalSze11}. Let $\ell_i$ denote the $i$th row of $L$. Without loss of generality, assume that each row of $H$ contains unique sets of symbols. Let $\sigma_1,\ldots,\sigma_{s_i}$ be the list of symbols in the $i$th row of $H$. The signal matrix $S_i\in\{0,1\}^{s_i\times M}$ is defined by $S_i (k,j) = \ind{H_{i,j}=\sigma_k}$ where $\ind{}$ is the indicator function. For a pair $i,k$ of actions define $S_{(i,k)}\in\{0,1\}^{(s_i+s_k)\times M}$ by stacking $S_i$ on top of $S_k$. Note that, upon playing action $i$, the signal $H_{i,j}$ arising from the unobserved action $j$ is equivalent to the feedback $S_i e_j$. 

Let $\C = \{C_1,\ldots,C_N\}$ be a partition of the simplex $\Delta_M$ according to the best response (action) of the player to the mixed strategy of the adversary:
	$$ C_i = \{q\in \Delta_M: i \mbox{ is best response for } q \} .$$
We assume that no action is completely dominated by others; that is, each $C_i$ is non-empty. Further, for simplicity we assume that $\C$ is indeed a partition and there are no degeneracies (we can modify the argument by defining  neighborhood action sets as in \cite{BarPalSze11}). \emph{Neighboring actions} are naturally defined as those that share a boundary in the partition. Let $\G$ be the graph obtained by connecting the neighboring cells of the partition $\C$. The vertex set of $\G$ is precisely the set $\{1,\ldots,N\}$ of player's actions. For each action $i$, let the set of its neighbors $N_i$ be called the {\it neighbor set}. By convention, any vertex is its own neighbor: $i\in N_i$. We will often use the terms \emph{action} and \emph{vertex} interchangeably, thanks to the one-to-one correspondence.

\begin{definition}[Bart\'ok, P\'al, Szepesv\'ari \cite{BarPalSze11}]
	The game is called locally observable if $\ell_i-\ell_j \in \text{Im}~S^\tr_{(i,j)}$ for all neighboring actions $i,j$.
\end{definition}

Under the local observability condition, for each pair of local actions $i,j$ there exists a vector $v_{(i,j)}$ such that $\ell_j-\ell_i = S_{(i,j)}^\tr v_{(i,j)}$. Since $L$ and $H$ are known, we can compute vectors $v_{(i,j)}$ and use them to construct unbiased estimates of true loss differences. 

\paragraph{Notation} Let $[N]$ denote the set $\{1,\ldots,N\}$. For a subset $S\subset [N]$ we use $\bone_{S} \in \{0,1\}^N$ to denote the vector with ones on the coordinates in $S$ and zeros outside. A vector $a\in \reals^N$ indexed by $j$ is sometimes denoted by $[a_j]_{j\in [N]}$. The scalar product between two vectors $a$ and $b$ will be variously written as $a^\tr b$ or $a\cdot b$. Standard basis vectors are denoted by $\{e_i\}$.

\section{Internal Regret in the Neighborhood}

Let $\phi:\{1,\ldots,N\}\mapsto \{1,\ldots,N\}$ be a \emph{departure function} \cite{CesLugSto06}, and let $i_t$ and $j_t$ denote the moves at time $t$ of the player and the opponent, respectively. At the end of the game, regret with respect to $\phi$ is calculated as the difference of the incurred cumulative cost and the cost that would have been incurred had we played action $\phi(i_t)$ instead of $i_t$, for all $t$. Let $\Phi$ be a set of departure functions. $\Phi$-regret is defined as 
$$ \frac{1}{T}\sum_{t=1}^T c(i_t, j_t) -  \inf_{\phi\in\Phi}\frac{1}{T}\sum_{t=1}^T c(\phi(i_t), j_t)$$
where the cost function considered in this paper is simply $c(i,j) = e_i^\tr L e_j$. If $\Phi=\{\phi_k: k\in[N]\}$ consists of constant mappings $\phi_k(i)=k$, the regret is called \emph{external}. For (global) internal regret, the set $\Phi$ consists of all  departure functions $\phi_{i\to j}$ such that $\phi_{i\to j} (i)=j$ and $\phi_{i\to j} (h)=h$ for $h\neq i$.

\begin{definition}
	\label{def:local_departure}
	A departure function $\phi_{i\to j}$ is called \emph{local departure function} if $j$ is a neighbor of $i$ in the neighborhood graph $\G$. Regret defined with respect to the set of all local departure functions is called \emph{local internal regret}.
\end{definition}

Under the local observability condition, we can estimate the differences in performance between the action and its neighbors in a way similar to non-stochastic bandit methods. We can, therefore, ensure that any time we chose an action, its loss was not much more than that of any of its neighbors. That is, local observability condition leads to an algorithm with no external regret and, under the flow condition detailed later, no local internal regret. A key observation is that no local internal regret implies no global internal regret. Intuitively, this stems from the fact that the second-best-response action must be a neighbor of the best-response action. Hence, ensuring small internal regret against the neighbors is enough to guarantee small internal regret.

\begin{figure}[htbp]
	\centering
		\includegraphics[height=1.8in]{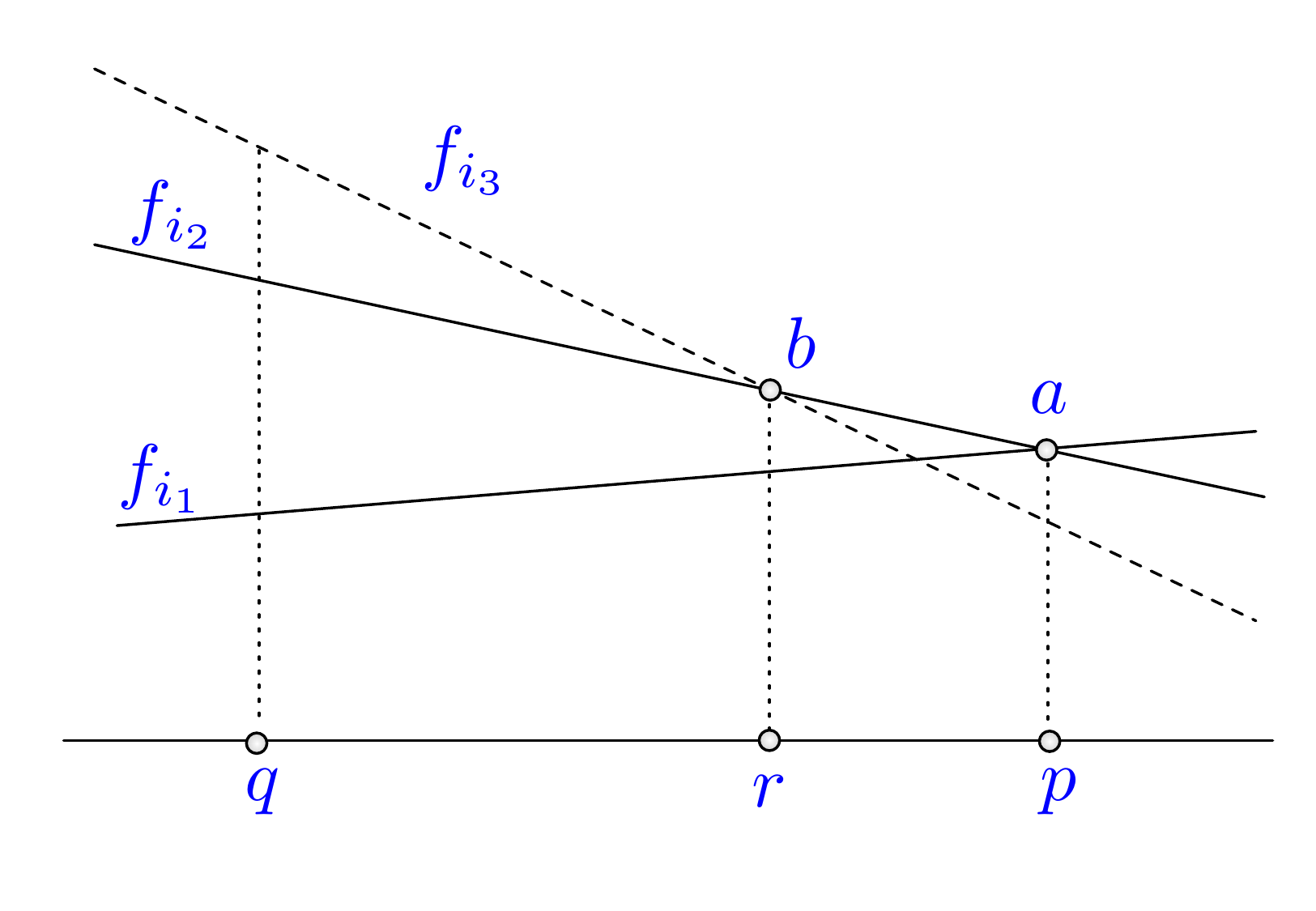}
	\label{fig:neighbors}
	\caption{Illustration of the argument in Lemma~\ref{lem:local_equivalent}: A second-best action must either be a neighbor, or it must be dominated everywhere by other actions.}
\end{figure}
\begin{lemma}
	\label{lem:local_equivalent}
	Local internal regret is equal to internal regret. 
\end{lemma}
\begin{proof}
	It is enough to show that, for any distribution $q\in \Delta_M$, any best response $i_1$ and any second-best response $i_2$ are neighbors in the graph $\G$. By the way of contradiction, we assume that actions $i_1$ and $i_2$ are not neighbors (that is, $C_{i_1}$ and $C_{i_2}$ do not share a face). We will then arrive at the conclusion that $i_2$ must be dominated by other actions, which is a contradiction because of our assumption that no action is completely dominated (that is minorized) by others. 
	
	Let $g(s) = \min_{i\in[N]} e_i^\tr Ls$ be the minimum loss against the mixed strategy $s$. Since $g$ is a minimum of $N$ linear functions $\{f_k (s) \deq (e_k^\tr L)\cdot s\}_{k=1}^N$, it is concave and piece-wise linear. The linear parts of $g$ correspond to the elements of the partition $\C$. By our assumption, $f_{i_1}(q)< f_{i_2}(q)$ and there is no hyperplane $f_{i_3}$ achieving at $q$ a value in the interval $(f_{i_1}(q), f_{i_2}(q))$. Let 
	$$S = \{(s,t)\in \reals^{M+1}: t=f_{i_1}(s)=f_{i_2}(s) ~\mbox{for some}~ s\in\Delta_M\},$$ 
	the intersection of two hyperplanes over the simplex. Note that projection of $S$ onto the simplex would be  precisely the boundary separating $C_{i_1}$ and $C_{i_2}$ if these were the only two actions. This set cannot be empty, for otherwise action $i_2$ is dominated by $i_1$. Now, pick any $p\in\Delta_M$ such that $f_{i_1}(p) = f_{i_2}(p)$, and let $a = (p,f_{i_1}(p))$ (see Figure~\ref{fig:neighbors}). We will now work with the one-dimensional problem along the line in the simplex defined by $(q,p)$. The fact that $i_1$ and $i_2$ are not neighbors along the direction $(q,p)$ means that there is another action $i_3$ such that $f_{i_3}(p) < f_{i_1}(p)=f_{i_2}(p)$. Since $f_{i_3}(q) \geq f_{i_2}(q) > f_{i_1}(q)$, there must be a point $b=(r, f_{i_3}(r))= (r, f_{i_2}(r))$ of intersection of $f_{i_3}$ and $f_{i_2}$ for some $r\in [q,p]$. It is easy to see that $i_2$ is completely minorized along the direction $(q,p)$: on one side of $r$ it is dominated by $i_1$, while on the other --- by $i_3$. 
	
The argument above works for any direction from $q$ towards the boundary between $C_{i_1}$ and $C_{i_2}$ if $i_1$ and $i_2$ were the only actions. Hence, $i_2$ is globally dominated by other actions, a contradiction.
\end{proof}

\section{Method}

The method is a two-level procedure motivated by Foster and Vohra \cite{FosVoh97} and Blum and Mansour \cite{BluMan07}. The intuition stems from the following observation. Suppose for each vertex $i$ we have a distribution $q_i\in\Delta_N$ supported on the neighbor set $N_i$. Let $p\in \Delta_N$ be defined by $p=Q p$ where $Q$ is the matrix $[q_1,\ldots,q_N]$. Then there are two equivalent ways of sampling an action from $p$. First way is to directly sample the vertex according to $p$. Second is to sample a vertex $i$ according to $p$ and then choose a vertex $j$ within the neighbor set $N_i$ according to $q_i$. Because of the stationarity (or \emph{flow}) condition $p = Qp$, the two ways are equivalent. This idea of finding a fixed point is implicit in \cite{FosVoh97}, and Blum and Mansour \cite{BluMan07} show how stationarity can be used to convert external regret guarantees into an internal regret statement. We show here that, in fact, this conversion can be done ``locally'' and only with ``comparison'' information between neighboring actions.

\begin{figure}[t]
	\centering
		\includegraphics[height=1.6in]{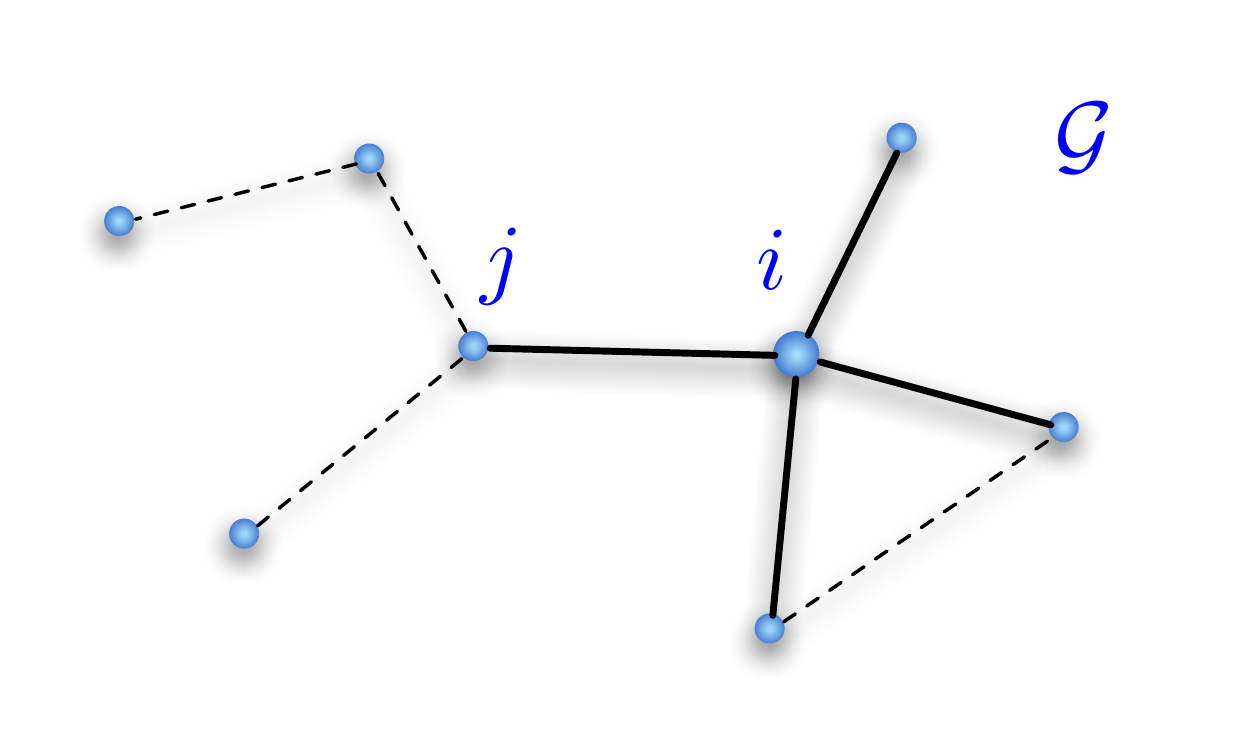}
	\caption{To each vertex $i$ in the graph $\G$ we associate an algorithm $\A_i$. The algorithm plays an action from the distribution $q^t_i$ over its neighborhood set $N_i$ and receives partial information about relative loss between the node $i$ and its neighbor. The other piece of the partial information comes from the times when a neighboring algorithm $\A_j$ is run and the action $i$ is picked.}
	\label{fig:neighborhood_graph}
\end{figure}

\begin{algorithm}[t]
\caption{\algoname{} Algorithm}
\label{alg}
\begin{algorithmic}[1]
\STATE For all $i=\{1,\ldots,N\}$, initialize algorithm $\A_i$ with $q^1_i = x^1_i = {\mathbf 1}_{N_i}/|N_i|$ 
\FOR{t=1,\ldots, T}
	\STATE Let $Q^t = [q^{t}_1,\ldots,q^t_N]$, where $q^t_i$ is furnished by $\A_i$
	\STATE Find $p^t$ satisfying $p^t = Q^t p^t$
	\STATE Draw $k_t$ from $p^t$
	\STATE Play $I_t$ drawn from $q^t_{k_t}$ and obtain signal $S_{I_t} e_{j_t}$
	\STATE Run local algorithm $\A_{k_t}$ with the received signal 
	\STATE For any $i\neq k_t$, $q^{t+1}_i \leftarrow q^t_i$
\ENDFOR
\end{algorithmic}
\end{algorithm}

\begin{algorithm}
\caption{Local Algorithm $\A_i$}
\label{alg_local}
\begin{algorithmic}[1]
	\STATE If $t=1$, initialize $s = 1$
	\STATE For $r \in \{\tau_i(s-1)+1,\ldots,\tau_i(s)\}$ (i.e. for all $r$ since the last time $\A_i$ was run) construct
	$$ b^r_{(i,j)} = v_{i,j}^\tr \left[ \begin{array}{c}
	 \ind{I_r = i} S_i \\
	 \ind{k_r=i}\ind{I_r = j} S_j/q^r_i(j)
	\end{array} \right] e_{j_r}$$
	for all $j\in N_i$
	\STATE Define for all $j\in N_i$,
			$$h_{(i,j)}^s = \sum_{r=\tau_i(s-1) +1}^{\tau_i(s)} b_{(i,j)}^r$$	
	and let
	$$\tilde{f}^s_i =  \left[ h^s_{(i,j)}\cdot \ind{j\in N_i} \right]_{j\in[N]} $$ 
	\STATE Pass the cost $\tilde{f}^s_i$ to a full-information online convex optimization algorithm over the simplex (e.g. Exponential Weights Algorithm) and receive the next distribution $x^{s+1}$ supported on $N_i$
	\STATE Define $$q^{t+1}_i \leftarrow (1-\gamma)x^{s+1} + (\gamma/|N_i|) \bone_{N_i}$$
	\STATE Increase the count $s \leftarrow s+1$
\end{algorithmic}
\end{algorithm}

Our procedure is as follows. We run $N$ different algorithms $\A_1,\ldots,\A_N$, each corresponding to a vertex and its neighbor set. Within this neighbor set we obtain small regret because we can construct estimates of loss differences  among the actions, thanks to the local observability condition. Each algorithm $\A_i$ produces a distribution $q^t_i \in \Delta_N$ at round $t$, reflecting the relative performance of the vertex $i$ and its neighbors. Since $\A_i$ is only concerned with its local neighborhood, we require that $q^t_i$ has support on $N_i$ and is zero everywhere else. The meta algorithm \algoname{} combines the distributions $Q^t=[q^t_1,\ldots,q^t_N]$ and computes $p^t$ as a fixed point 
\begin{align}
	\label{eq:stationarity}
	p^t = Q^t p^t \ . 
\end{align}

How do we choose our actions? At each round, we draw $k_t \sim p_t$ and then $I_t\sim q^t_{k_t}$ according to our two-level scheme. The action $I_t$ is the action we play in the partial monitoring game against the adversary. Let the action played by the adversary at time $t$ be denoted by $j_t$. Then the feedback we obtain is $S_{I_t} e_{j_t}$. This information is passed to $\A_{k_t}$ which updates the distributions $q^t_{k_t}$. In Section~\ref{sec:estimation} we detail how this is done.

\subsection{Main Result}

The main result of the paper is the following internal regret guarantee.
\begin{theorem}
	\label{thm:regret}
	Local internal regret of Algorithm~\ref{alg} is bounded as
	\begin{align*} 
		\sup_{\phi}\E \left\{ \sum_{t=1}^T (e_{I_t}-e_{\phi(I_t)})^\tr L e_{j_t} \right\} \leq 4N\bar{v}\sqrt{6 (\log N) T} 
	\end{align*}
	where $\bar{v} = \max_{(i,j)} \|v_{(i,j)}\|_\infty$ and supremum is taken over all local departure functions. 
\end{theorem}

The next Corollary is immediate given Lemma~\ref{lem:local_equivalent}:

\begin{corollary}
	Internal regret of Algorithm~\ref{alg} is also bounded as in Theorem~\ref{thm:regret}.
\end{corollary}

We remark that high probability bounds can also be obtained in a rather straightforward manner, using, for instance, the approach of \cite{AbeRak09colt}. Another extension, the case of random signals, is discussed in Section~\ref{sec:randomsignals}.

\subsection{Estimating loss differences}
\label{sec:estimation}
The random variable $k_t$ drawn from $p^t$ at time $t$ determines which algorithm is active on the given round. Let 
$$\tau_i(s) = \min\{t~:~s = \sum_{r=1}^t \ind{k_t = i} \}$$ 
denote the (random) time when the algorithm $\A_i$ is invoked for the $s$-th time. By convention, $\tau_i(0) = 0$. Further, define 
$$\pi_i(t) = \min\{t'\geq t~:~ k_{t'} = i \}$$
to denote the next time the algorithm is run on or after time $t$. When invoked for the $s$-th time, the algorithm $\A_{i}$ constructs estimates 
$$ b^r_{(i,j)} \deq v_{i,j}^\tr \left[ \begin{array}{c}
 \ind{I_r = i} S_i \\
 \ind{k_r=i}\ind{I_r = j} S_j/q^r_i(j)
\end{array} \right] e_{j_r}  ~, ~~~~~~ \forall r \in \{\tau_i(s-1)+1,\ldots,\tau_i(s)\}, ~\forall j\in N_i $$
for all the rounds  after it has been run the last time, until (and including) the current time $r=\tau_i(s)$. We can assume  $b^t_{(i,j)}=0$ for any $j\notin N_i$. The estimates $b^t_{(i,j)}$ can be constructed by the algorithm because $S_{I_r} e_{j_r}$ is precisely the feedback given to the algorithm. 

Let $\F_t$ be the $\sigma$-algebra generated by the random variables $\{k_1,I_1,\ldots,k_t,I_t\}$. For any $t$, the (conditional) expectation, 
\begin{align}
	\label{eq:unbiased_b}
	\E \left[ b^t_{(i,j)} | \F_{t-1} \right] &= \sum_{k=1}^N p^t_k q^t_k(i) \cdot v_{i,j}^\tr \left[ \begin{array}{c}
 S_i  \nonumber \\
 0 
\end{array} \right] e_{j_t} + p^t_i q^t_i(j) \cdot v_{i,j}^\tr \left[ \begin{array}{c}
 0 \\
 S_j/q^t_i(j))
\end{array} \right] e_{j_t} \nonumber  \\
&= p^t_i v_{i,j}^\tr S_{(i,j)} e_{j_t} \nonumber  \\
&= p^t_i (\ell_j-\ell_i)^\tr e_{j_t} \nonumber \\ 
&= p^t_i (e_j-e_i)^\tr L e_{j_t}
\end{align}
where in the second equality we used the fact that $\sum_{k=1}^N p^t_k q^t_k(i) = p^t_i$ by stationarity \eqref{eq:stationarity}.
Thus each algorithm $\A_i$, on average, has access to unbiased estimates of the loss differences within its neighborhood set. 

Recall that algorithm $\A_i$ is only aware of its neighborhood, and therefore we peg coordinates of $q^t_i$ to zero outside of $N_i$. However, for convenience, our notation below still employs full  $N$-dimensional vectors, and we keep in mind that only coordinates indexed by $N_i$ are considered and modified by $\A_i$. 

When invoked for the $s$-th time (that is, $t=\tau_i(s)$), $\A_i$ constructs linear functions (cost estimates) $\tilde{f}^s_i \in \reals^N$ defined by 
$$\tilde{f}^s_i =  \left[ h^s_{(i,j)}\cdot \ind{j\in N_i} \right]_{j\in[N]}, $$
where
$$h_{(i,j)}^s = \sum_{r=\tau_i(s-1) +1}^{\tau_i(s)} b_{(i,j)}^r \ .$$

We now show that $\tilde{f}^s_i \cdot q^{\tau(s)}_i$ has the same conditional expectation as the actual loss of the meta algorithm \algoname{} at time $t=\tau_i(s)$. That is, by bounding expected regret of the black-box algorithm operating on $\{\tilde{f}^s_i\}$, we bound the actual regret suffered by the meta algorithm on the rounds when $\A_i$ was invoked.

\begin{lemma}	
	\label{lem:unbiasedness}
	Consider algorithm $\A_i$. It holds that
	$$ \E\left\{ (q^{\tau_i(s+1)}_i-e_u)^\tr L e_{j_{\tau_i(s+1)}} ~\middle|~ \F_{\tau_i(s)}\right\} = \E\left\{ \tilde{f}^{s+1}_i \cdot (q^{\tau_i(s+1)}_i-e_u) ~\middle|~ \F_{\tau_i(s)}\right\}$$
	for any $u\in N_{i}$.
\end{lemma}
\begin{proof}
	Throughout the proof, we drop the subscript $i$ on $\tau_i$ to ease the notation. Note that $q^{\tau(s+1)}_i = q^{\tau(s)+1}_i$ since the distribution is not updated when algorithm $\A_i$ is not invoked. Hence, conditioned on  $\F_{\tau(s)}$, the variable $(q^{\tau(s+1)}_i-e_u)$ can be taken out of the expectation. We therefore need to show that 
	\begin{align} 
		\label{eq:desired_eq}
		(q^{\tau(s+1)}_i-e_u)\cdot \E\left\{ L e_{j_{\tau(s+1)}}| \F_{\tau(s)}\right\} =  (q^{\tau(s+1)}_i-e_u)\cdot  \E\left\{ \tilde{f}^{s+1}_i | \F_{\tau(s)}\right\}
	\end{align}
First, we can write
\begin{align*}
	\E\left\{ h^{s+1}_{(i,j)} ~~\middle|~~ \F_{\tau(s)}\right\} &= \E\left\{ \sum_{t=\tau(s)+1}^{\tau(s+1)} b^{t}_{(i,j)} ~~\middle|~~\F_{\tau(s)}\right\}\\
	&= \E\left\{ \sum_{t=\tau(s)+1}^{\infty} b^{t}_{(i,j)} \ind{t\leq \tau(s+1)} ~~\middle|~~\F_{\tau(s)}\right\}\\
	&= \sum_{t=\tau(s)+1}^{\infty} \E\left\{ \E\left[  b^{t}_{(i,j)} \ind{t\leq \tau(s+1)} ~\middle|~ \F_{t-1} \right] ~~\middle|~~\F_{\tau(s)}\right\} \\
	&= \sum_{t=\tau(s)+1}^{\infty} \E\left\{ \ind{t\leq \tau(s+1)} \E\left[  b^{t}_{(i,j)} ~\middle|~ \F_{t-1} \right] ~~\middle|~~\F_{\tau(s)}\right\} \ .
\end{align*}
The last step follows because the event $\{t\leq \tau(s+1)\}$ is $\F_{t-1}$-measurable (that is, variables $k_1,\ldots,k_{t-1}$ determine the value of the indicator). By Eq.~\eqref{eq:unbiased_b}, we conclude
\begin{align}
	\label{eq:interm}
	\E\left\{ h^{s+1}_{(i,j)} ~~\middle|~~ \F_{\tau(s)}\right\} &= \sum_{t=\tau(s)+1}^{\infty} \E\left\{ \ind{t\leq \tau(s+1)} p^t_i (e_j-e_i)^\tr L e_{j_t} ~~\middle|~~\F_{\tau(s)}\right\} \ .
\end{align}
Since $\ind{t=\tau(s+1)} = \ind{k_t=i}\ind{t\leq \tau(s+1)}$, we have
\begin{align*}
\E \left\{ \ind{t=\tau(s+1)} e_{j_t} ~~\middle|~~ \F_{\tau(s)} \right\} &= \E\left\{\E\left\{ \ind{k_t=i}\ind{t\leq \tau(s+1)}e_{j_t} ~~\middle|~~ \F_{t-1}  \right\} ~~\middle|~~ \F_{\tau(s)} \right\} \\
&= \E\left\{\ind{t\leq \tau(s+1)} e_{j_t} \E\left\{ \ind{k_t=i} ~~\middle|~~ \F_{t-1}  \right\} ~~\middle|~~ \F_{\tau(s)}  \right\} \\
&= \E\left\{\ind{t\leq \tau(s+1)} \P(k_t=i ~\middle|~ \F_{t-1}) e_{j_t} ~~\middle|~~ \F_{\tau(s)}  \right\} \\
&= \E\left\{\ind{t\leq \tau(s+1)} p^t_i e_{j_t} ~~\middle|~~ \F_{\tau(s)}  \right\} .
\end{align*}
Combining with Eq.~\eqref{eq:interm},
\begin{align*}
	\E\left\{ h^{s+1}_{(i,j)} ~~\middle|~~ \F_{\tau(s)}\right\} &= \sum_{t=\tau(s)+1}^{\infty} \E\left\{ \ind{t\leq \tau(s+1)} p^t_i (e_j-e_i)^\tr L e_{j_t} ~~\middle|~~\F_{\tau(s)}\right\} \\
	&= \sum_{t=\tau(s)+1}^{\infty} \E\left\{ \ind{t= \tau(s+1)} (e_j-e_i)^\tr L e_{j_t} ~~\middle|~~\F_{\tau(s)}\right\}
\end{align*}
Observe that coordinates of $\tilde{f}^{s+1}_i$, $q^{\tau(s+1)}_i$, and $e_u$ are zero outside of $N_{i}$. We then have that 
\begin{align*}
	\E\left\{ \tilde{f}^{s+1}_{i} ~~\middle|~~ \F_{\tau(s)}\right\} &= 	\left[ \ind{j\in N_i} \E\left\{ h^{s+1}_{(i,j)} ~~\middle|~~ \F_{\tau(s)}\right\} \right]_{j\in N} \\
	&= \left[\ind{j\in N_i} \sum_{t=\tau(s)+1}^{\infty} \E\left\{ (e_{j}-e_{i})^\tr L e_{j_t} \ind{ t = \tau(s+1) }  ~~\middle|~~\F_{\tau(s)}\right\}  \right]_{j\in N} \\
	&= \left[\ind{j\in N_i} \sum_{t=\tau(s)+1}^{\infty} \E\left\{ e_{j} L e_{j_t} \ind{ t = \tau(s+1) }  ~~\middle|~~\F_{\tau(s)}\right\}  \right]_{j\in N} - c \cdot \bone_{N_i}
\end{align*}
where $$c = \sum_{t=\tau(s)+1}^{\infty} \E\left\{ e_{i} L e_{j_t} \ind{ t = \tau(s+1) }  ~~\middle|~~\F_{\tau(s)}\right\}$$ is a scalar. When multiplying the above expression by $q^{\tau(s+1)}_i-e_u$, the term $c \cdot \bone_{N_i}$ vanishes. Thus, minimizing regret with relative costs (with respect to the $i$th action) is the same as minimizing regret with the absolute costs. We conclude that
\begin{align*}
	(q^{\tau(s+1)}_i-e_u) \E\left\{ \tilde{f}^{s+1}_{i} ~~\middle|~~ \F_{\tau(s)}\right\} &= (q^{\tau(s+1)}_i-e_u)\cdot \left[\sum_{t=\tau(s)+1}^{\infty} \E\left\{ e_{j} L e_{j_t} \ind{ t = \tau(s+1) }  ~~\middle|~~\F_{\tau(s)}\right\}  \right]_{j\in N_{i}}\\
	&= (q^{\tau(s+1)}_i-e_u) \cdot \sum_{t=\tau(s)+1}^{\infty} \E\left\{ Le_{j_t} \ind{ t = \tau(s+1) }  ~~\middle|~~\F_{\tau(s)}\right\} \\
	&= (q^{\tau(s+1)}_i-e_u)\cdot \E\left\{ L e_{j_{\tau(s+1)}} ~~\middle|~~ \F_{\tau(s)}\right\}
\end{align*}
\end{proof}

\subsection{Regret Analysis}
\label{sec:analysis}

For each algorithm $\A_i$, the estimates $\tilde{f}^s_i$ are passed to a full-information black box algorithm which works only on the coordinates $N_i$. From the point of view of the full-information black box, the game has length $T_i = \max\{s: \tau_i(s)\leq T\}$, the (random) number of times action $i$ has been played within $T$ rounds.

We proceed similarly to \cite{AbeRak09colt}: we use a full-information online convex optimization procedure with an entropy regularizer (also known as the Exponential Weights Algorithm) which receives the vector $\tilde{f}^s_i$ and returns the next mixed strategy $x^{s+1} \in \Delta_N$ (in fact, effectively in $\Delta_{|N_i|}$). We then define 
$$q^{t+1}_i = (1-\gamma)x^{s+1} + (\gamma/|N_i|) \bone_{N_i}$$ 
where $\gamma$ is to be specified later. Since $\A_i$ is run at time $t$, we have $\tau_i(s)=t$ by definition. The next time $\A_i$ is active (that is, at time $\tau_i(s+1)$), the action $I_{\tau_i(s+1)}$ will be played as a random draw from $q^{t+1}_i = q^{\tau_i(s+1)}_i$; that is, the distribution is not modified on the interval $\{\tau_i(s)+1,\ldots,\tau_i(s+1)\}$.

We prove Theorem~\ref{thm:regret} by a series of lemmas. The first one is a direct consequence of an external regret bound for a Follow the Regularized Leader (FTRL) algorithm in terms of local norms \cite{AbeRak09colt}.	For a strictly convex ``regularizer'' $F$, the local norm $\|\cdot\|_x$ is defined by $\|z\|_x = \sqrt{z^\tr \nabla^2 F(x) z}$ and its dual is $\|z\|^*_x = \sqrt{z^\tr \nabla^2 F(x)^{-1} z}$. 

\begin{lemma} 
	\label{lem:full_info_applied}
	The full-information algorithm utilized by $\A_i$ has an upper bound 
	$$\E\left\{\sum_{s=1}^{T_i} \tilde{f}^s_i \cdot (q^{\tau_i(s)}_i - e_{\phi(i)}) \right\}\leq \eta \E\left\{ \sum_{s=1}^{T_i} (\|\tilde{f}^s_i\|^*_{x^s})^2 \right\} + \eta^{-1} \log N + T\gamma\bar{\ell}$$
	on its external regret, where $\phi(i)\in N_i$ is any neighbor of $i$, $\bar{\ell} = \max_{i,j} L_{i,j}$, and $\eta$ is a learning rate parameter to be tuned later.
\end{lemma}
\begin{proof}
	Since our decision space is a simplex, it is natural to use the (negative) entropy regularizer, in which case FTRL is the same as the Exponential Weights Algorithm. From \cite[Thm 2.1]{AbeRak09colt}, for any comparator $u$ with zero support outside $|N_i|$, the following regret guarantee holds:
	$$\sum_{s = 1}^{T_i} \tilde{f}^s_i \cdot (x^s-u) \leq \eta \sum_{s=1}^{T_i} (\|\tilde{f}^s_i\|^*_{x^s})^2 + \eta^{-1} \log(|N_i|) \ .$$ 
	An easy calculation shows that in the case of entropy regularizer $F$, the Hessian $\nabla^2 F(x) = \text{diag}(x_1^{-1},x_2^{-1},\ldots,x_N^{-1})$ and $\nabla^2 F(x)^{-1} = \text{diag}(x_1,x_2,\ldots,x_N)$. We refer to \cite{AbeRak09colt} for more details.
	
	Let $\phi:\{1,\ldots,N\}\mapsto\{1,\ldots,N\}$ be a local departure function (see Definition~\ref{def:local_departure}). We can then write a regret guarantee 
	$$\sum_{s=1}^{T_i} \tilde{f}^s_i \cdot (x^s - e_{\phi(i)}) \leq \eta \sum_{s=1}^{T_i} (\|\tilde{f}^s_i\|^*_{x^s})^2 + \eta^{-1} \log(|N_i|) \ .$$
	Since, in fact, we play according to a slightly modified version $q^{\tau_i(s)}_i$ of $x^s$, it holds that
	$$\sum_{s=1}^{T_i} \tilde{f}^s_i \cdot (q^{\tau_i(s)}_i - e_{\phi(i)}) \leq \eta \sum_{s=1}^{T_i} (\|\tilde{f}^s_i\|^*_{x^s})^2 + \eta^{-1} \log(|N_i|) + \sum_{s=1}^{T_i} \tilde{f}^s_i \cdot (q^{\tau_i(s)}_i - x^s) \ .$$ 
	Taking expectations of both sides and upper bounding $|N_i|$ by $N$,
	$$\E\left\{\sum_{s=1}^{T_i} \tilde{f}^s_i \cdot (q^{\tau_i(s)}_i - e_{\phi(i)}) \right\} \leq \eta \E\left\{ \sum_{s=1}^{T_i} (\|\tilde{f}^s_i\|^*_{x^s})^2 \right\} + \eta^{-1} \log N + \E\left\{ \sum_{s=1}^{T_i} \tilde{f}^s_i \cdot (q^{\tau_i(s)}_i - x^s) \right\} \ .$$
	A proof identical to that of Lemma~\ref{lem:unbiasedness} gives
	\begin{align*} 
		\E\left\{ \tilde{f}^{s}_i \cdot (q^{\tau_i(s)}_i-x^s) ~\middle|~ \F_{\tau_i(s-1)}\right\} &= \E\left\{ (q^{\tau_i(s)}_i-x^s)^\tr L e_{j_{\tau_i(s)}} | \F_{\tau_i(s-1)}\right\} \\
		&\leq \E\left\{ \|q^{\tau_i(s)}_i-x^s\|_1 \cdot \| L e_{j_{\tau_i(s)}}\|_\infty ~\middle|~ \F_{\tau_i(s-1)}\right\}  \\
		&\leq \gamma\bar{\ell} 
	\end{align*}
	for the last term, where $\bar{\ell}$ is the upper bound on the magnitude of entries of $L$. Putting everything together, 
	$$\E\left\{\sum_{s=1}^{T_i} \tilde{f}^s_i \cdot (q^{\tau_i(s)}_i - e_{\phi(i)}) \right\} \leq \eta \E\left\{ \sum_{s=1}^{T_i} (\|\tilde{f}^s_i\|^*_{x^s})^2 \right\} + \eta^{-1} \log N + T\gamma\bar{\ell} $$
	where we have upper bounded $T_i$ by $T$. 
\end{proof}

As with many bandit-type problems, effort is required to show that the variance term is controlled. This is the subject of the next lemma.
\begin{lemma} 
	\label{lem:variance}
	The variance term in the bound of Lemma~\ref{lem:full_info_applied} is upper bounded as
	$$\sum_{i=1}^N \E\left\{ \sum_{s=1}^{T_i} (\|\tilde{f}^s_i\|^*_{x^s})^2 \right\} \leq 24\bar{v}^2 N T$$	
\end{lemma}
\begin{proof}
	First, fix an $i\in[N]$ and consider the term $\E\left\{ \sum_{s=1}^{T_i} (\|\tilde{f}^s_i\|^*_{x^s})^2 \right\}$. Until the last step of the proof, we will sometimes omit $i$ from the notation.
	
	We start by observing that $\tilde{f}^s_i$ is a sum of $\tau(s)-\tau(s-1)-1$ terms of the type $v_{i,j}^\tr S_i e_{j_r}$ (that is, of constant magnitude) and one term of the type $v_{i,j}^\tr S_j e_{j_r} / q^r_i(j)$. In controlling $\|\tilde{f}^s_i\|^*_{x^s}$, we therefore have two difficulties: controlling the number of constant-size terms and making sure the last term does not explode due to division by a small probability $q^r_i(j)$. The former is solved below by a careful argument below, while the latter problem is solved according to usual bandit-style arguments.

More precisely, we can write $\tilde{f}^s_i = g^{\tau_i(s-1)}_{\tau_i(s)}+h^{\tau_i(s)}$ where the vectors $g^{\tau_i(s-1)}_{\tau_i(s)}, h^{\tau_i(s)}\in\reals^N$ are defined as 
$$g^{\tau_i(s-1)}_{\tau_i(s)}(j) \deq g^{\tau_i(s-1)}(j) \deq \sum_{r = \tau_i(s-1)}^{\tau_i(s)-1} \ind{I_r = i} v_{i,j}^\tr  S_i  e_{j_r} \ind{j\in N_i}$$ 
and
$$h^{\tau_i(s)}(j) = \ind{I_{\tau_i(s)}=j} v_{i,I_{\tau_i(s)}}^\tr S_{I_{\tau_i(s)}} e_{j_{\tau_i(s)}}/q^{\tau_i(s)}_i(I_{\tau_i(s)}) \ .$$
Then
$$(\|\tilde{f}^s_i\|^*_{x^s})^2 = (\|g^{\tau_i(s-1)}+h^{\tau_i(s)}\|^*_{x^s})^2 \leq 2(\|g^{\tau_i(s-1)} \|^*_{x^s})^2 + 2(\|h^{\tau_i(s)} \|^*_{x^s})^2 $$
We will bound each of the two terms separately, in expectation. For the second term,
$$(\|h^{\tau_i(s)} \|^*_{x^s})^2 = x^s(I_\tau) (v_{i,I_\tau}^\tr S_{I_\tau} e_{j_\tau}/q^{\tau}_i(I_\tau))^2 \leq  x^s(I_\tau) (\bar{v}/q^{\tau}_i(I_\tau))^2$$
where $\tau = \tau_i(s)$. Since $q^{\tau_i(s)}_i = (1-\gamma)x^{s} + (\gamma/|N_i|) \bone_{N_i}$, it is easy to verify that $x^s(I_\tau)/ q^{\tau}_i(I_\tau) \leq 2$ (whenever $\gamma<1/2$) and thus 
$$(\|h^{\tau_i(s)} \|^*_{x^s})^2 \leq  2 \bar{v}^2/q^{\tau}_i(I_\tau) \ .$$
The remaining division by the probability disappears under the expectation:
\begin{align}
	\label{eq:bandit_term_control}
	\E\left\{ (\|h^{\tau_i(s)} \|^*_{x^s})^2 ~\middle|~ \sigma(k_1,I_1,\ldots, k_{\tau_i(s)} )\right\} \leq 2\bar{v}^2 \sum_{j=1}^N  q^{\tau_i(s)}_i(j) /q^{\tau_i(s)}_i(j) = 2N \bar{v}^2 \ .
\end{align}

Consider now the second term. As discussed in the proof of Lemma~\ref{lem:full_info_applied}, the inverse Hessian of the entropy function shrinks each coordinate $i$ precisely by $x^s(i)\leq 1$, implying that the local norm is dominated by the Euclidean norm :
$$\|g^{\tau_i(s-1)}  \|^*_{x^s} \leq \|g^{\tau_i(s-1)} \|_2 .$$
It is therefore enough to upper bound $\E\left\{ \sum_{s=1}^{T_i} \|g^{\tau_i(s)} \|_2^2 \right\}$. The idea of the proof is the following. Observe that $\P(k_{t} = i|\F_{t-1}) = \P(I_t = i|\F_{t-1})$. Conditioned on the event that either $k_t=i$ or $I_t = i$, each of the two possibilities has probability $1/2$ of occurring. Note that $g^{\tau_i(s-1)}$ inflates every time $k_t \neq i$, yet $I_t = i$ occurs. It is then easy to see that magnitude of $g^{\tau_i(s-1)}$ is unlikely to get large before algorithm $\A_i$ is run again. We now make this intuition precise.

The function $g^{t}$ is presently defined only for those time steps when $t=\tau_i(s)$ for some $s$ (that is, when the algorithm $\A_i$ is invoked). We extend this definition as follows. Let the $j$th coordinate of $g^t$ be defined as
$$ g^t_{\pi(t+1)} (j) \deq g^t (j) \deq \sum_{r=t}^{\pi(t+1)-1} \ind{I_r = i} v_{(i,j)} S_i e_{j_r}$$
for $j\in N_i$ and $0$ otherwise. The function $g^t$ can be thought of as accumulating partial pieces on rounds when $I_t=i$ until $k_t=i$ occurs. Let us now define an analogue of $\tau$ and $\pi$ for the event that \emph{either} $I_t=i$ or $k_t = i$:
$$\gamma_i(s) = \min\left\{t~:~s = \sum_{r=1}^t \ind{k_t = i ~\mbox{or}~ I_t = i} \right\}$$ 
Further, for any $t$, let $$\nu_i(t) = \min\{t'\geq t: k_t = i ~\mbox{or}~ I_t = i \},$$
the next time occurrence of the event $\{k_\tau = i ~\mbox{or}~ I_\tau = i\}$ on or after $t$. Let 
$${\mathcal I} = \ind{ \nu_i(t) \neq \pi_i(t) }$$ 
be the indicator of the event that the first time after $t$ that $\{k_\tau = i ~\mbox{or}~ I_\tau = i\}$ occurred it was also the case that the algorithm was not run (i.e. $k_\tau\neq i$). Note that $g^t(j)$ can now be written recursively as
$$ g^t(j) = {\mathcal I} \cdot \left[ v_{(i,j)} S_i e_{j_{\nu(t)}} + g^{\nu(t) +1}_{\pi(\nu(t)+1)}(j)\right].$$
As argued before, $\P({\mathcal I}=1 | \F_{t-1}) = 1/2$.
We will now show that $\E \left\{ g^t(j) ~\middle|~ \F_{t-1}\right\} \le 2\bar{v}$ by the following inductive argument, whose base case trivially holds for $t=T$:
\begin{align*}
	\E \left\{ g^t(j) ~\middle|~ \F_{t-1}\right\} 
	&= \E \left\{ \E \left\{ {\mathcal I} \cdot \left[ v_{(i,j)} S_i e_{j_{\nu(t)}} + g^{\nu(t) +1}(j)\right] ~\middle|~ \F_{\nu(t)} \right\} ~\middle|~ \F_{t-1} \right\} \\
	&= \E \left\{ {\mathcal I} v_{(i,j)} S_i e_{j_{\nu(t)}} + {\mathcal I} \E \left\{ g^{\nu(t) +1}(j) ~\middle|~ \F_{\nu(t)} \right\} ~\middle|~ \F_{t-1} \right\} \\
	&\leq \bar{v} + \E \left\{ {\mathcal I} g^{\nu(t) +1}(j) ~\middle|~ \F_{t-1} \right\} \\
	&= \bar{v} + \E \Big\{ {\mathcal I}  \underbrace{\E\left[ 
g^{\nu(t) +1}(j) ~\middle|~
\F_{\nu(t)}\right]}_{\hbox{$\le 2\bar{v}$ by induction}} ~\Big|~ \F_{t-1} \Big\} \\
	&\le \bar{v} + \E \left\{ {\mathcal I}  ~\middle|~ \F_{t-1} \right\} 2\bar{v} \\
	&\le \bar{v} + (1/2) 2\bar{v} = 2\bar{v}
\end{align*}
The expected value of $(g^t(j))^2$ can be controlled in a similar manner. To ease the notation, let $z = v_{(i,j)} S_i e_{j_{\nu(t)}}$. Using the upper bound for the conditional expectation of $g^t(j)$ calculated above,
\begin{align*}
	\E \left\{ (g^t(j))^2 ~\middle|~ \F_{t-1}\right\} 
	&= \E \left\{ {\mathcal I}\cdot \left( z^2 + (g^{\nu(t)+1}(j))^2 + 2 z g^{\nu(t)+1}(j) \right) ~\middle|~ \F_{t-1}\right\} \\
	&= \E \left\{ {\mathcal I} z^2 + {\mathcal I} \E\left\{(g^{\nu(t)+1}(j))^2 ~\middle|~ \F_{\nu(t)}\right\} + 2{\mathcal I}z \E\left\{ g^{\nu(t)+1}(j) ~\middle|~ \F_{\nu(t)}\right\} ~\middle|~ \F_{t-1}\right\} \\
	&\leq 5\bar{v}^2 + \E \left\{ {\mathcal I} \E\left\{ (g^{\nu(t)+1}(j))^2 ~\middle|~ \F_{\nu(t)} \right\}~\middle|~ \F_{t-1} \right\} 
\end{align*}
The argument now proceeds with backward induction exactly as above. We conclude that
$$\E \left\{ (g^t(j))^2 ~\middle|~ \F_{t-1}\right\} \leq 10\bar{v}^2$$
and, hence,
$$\E\left\{ \|g^{\tau_i(s-1)} \|_2^2 \right\}\leq 10N\bar{v}^2$$
Together with \eqref{eq:bandit_term_control}, we conclude that
$$\E\left\{ (\|\tilde{f}^s_i\|^*_{x^s})^2 \right\} \leq 2(2N\bar{v}^2 + 10N\bar{v}^2) = 24\bar{v}^2 N.$$	
Summing over $t=1,\ldots,T$ and observing that only one algorithm is run at any time $t$ proves the statement.
\end{proof}

\begin{proof}[\textbf{Proof of Theorem~\ref{thm:regret}}]

	The flow condition $p^t = Q^t p^t$ comes in crucially in several places throughout the proofs, and the next argument is one of them. Observe that
	\begin{align*}
		\E \left\{ e_{\phi(I_t)} ~\middle| \F_{t-1}\right\} &= \sum_{k=1}^N \sum_{i=1}^N p^t_k q^t_k(i) e_{\phi(i)} 
		= \sum_{i=1}^N e_{\phi(i)} \sum_{k=1}^N  p^t_k q^t_k(i)  
		= \sum_{i=1}^N e_{\phi(i)} p^t_i 
		= \E \left\{ e_{\phi(k_t)} ~\middle| \F_{t-1}\right\}
	\end{align*}
	and thus
	\begin{align*}
		\E \left\{ \sum_{t=1}^T e_{\phi(I_t)}^\tr L e_{j_t} \right\} 
		&= \E \left\{ \sum_{t=1}^T \E \left\{ e_{\phi(I_t)} ~\middle| \F_{t-1}\right\}^\tr L e_{j_t} \right\} \\
		&= \E \left\{ \sum_{t=1}^T \E \left\{ e_{\phi(k_t)} ~\middle| \F_{t-1}\right\}^\tr L e_{j_t} \right\} \\
		&= \E \left\{ \sum_{t=1}^T e_{\phi(k_t)}^\tr L e_{j_t} \right\} 
	\end{align*}
	It is because of this equality that external regret with respect to the local neighborhood can be turned into local internal regret. We have that 
\begin{align} 
	\label{eq:total_regret_bound}
	\E \left\{ \sum_{t=1}^T (e_{I_t}-e_{\phi(I_t)})^\tr L e_{j_t} \right\} &= \E \left\{ \sum_{t=1}^T (e_{I_t}-e_{\phi(k_t)})^\tr L e_{j_t} \right\} \nonumber \\
	&= \E \left\{ \sum_{t=1}^T (q^t_{k_t}-e_{\phi(k_t)})^\tr L e_{j_t} \right\} \nonumber \\
	&= \sum_{i=1}^N \E \left\{ \sum_{t=1}^T \ind{k_t=i} (q^t_{i}-e_{\phi(i)})^\tr L e_{j_t} \right\} \nonumber 
\end{align}
By Lemma \ref{lem:unbiasedness},
	$$\E\left\{ (q^{\tau_i(s)}_i-e_{\phi(i)})^\tr L e_{j_{\tau_i(s)}} | \F_{\tau_i(s-1)}\right\} = \E\left\{ \tilde{f}^{s}_i \cdot (q^{\tau_i(s)}_i-e_{\phi(i)}) ~\middle|~ \F_{\tau_i(s-1)}\right\}$$
and so by Lemma~\ref{lem:full_info_applied}
\begin{align*}
	E \left\{ \sum_{t=1}^T (e_{I_t}-e_{\phi(I_t)})^\tr L e_{j_t} \right\}	
	&= \sum_{i=1}^N \E \left\{ \sum_{s=1}^{T_i} \tilde{f}^{s}_i \cdot (q^{\tau_i(s)}_i-e_{\phi(i)})  \right\} \nonumber \\
	&\leq \eta\sum_{i=1}^N \E\left\{ \sum_{s=1}^{T_i} (\|\tilde{f}^s_i\|^*_{x^s})^2 \right\} + N(\eta^{-1}\log N + T\gamma\bar{\ell} ) 
\end{align*}
With the help of Lemma~\ref{lem:variance},
\begin{align*} 
	\E \left\{ \sum_{t=1}^T (e_{I_t}-e_{\phi(I_t)})^\tr L e_{j_t} \right\} &\leq \eta 24\bar{v}^2 N T  + N(\eta^{-1}\log N + T\gamma\bar{\ell} ) 
	= 4N\bar{v}\sqrt{6 (\log N) T} + TN\gamma\bar{\ell}
\end{align*}
for the setting of $\eta = \sqrt{\frac{\log N}{24\bar{v}^2 T}}$.

We remark that for the purposes of ``in expectation'' bounds, we can simply set $\gamma=0$ and still get $O(\sqrt{T})$ guarantees (see \cite{AbeRak09colt}). This point is obscured by the fact that the original algorithm of Auer et al \cite{auer2003nonstochastic} uses the same parameter for the learning rate $\eta$ and exploration $\gamma$. If these are separated, the ``in expectation'' analysis of \cite{auer2003nonstochastic} can be also done with $\gamma=0$. However, to prove high probability bounds on regret, a setting of $\gamma\propto T^{-1/2}$ is required. Using the techniques in \cite{AbeRak09colt}, the high-probability extension of results in this paper is straightforward (tails for the terms $\|g^{\tau_i(s-1)} \|_2^2$ in Lemma~\ref{lem:variance} can be controlled without much difficulty).
\end{proof}

\section{Random Signals}
\label{sec:randomsignals}

We now briefly consider the setting of partial monitoring with random signals, studied by Rustichini \cite{Rustichini99}, Lugosi, Mannor, and Stoltz \cite{LugManSto08}, and Perchet \cite{Perchet11}. Without much modification of the above arguments, the local observability condition yet again yields $O(\sqrt{T})$ internal regret. 

Suppose that instead of receiving deterministic feedback $H_{i,j}$, the decision maker now receives a random signal $d_{i,j}$ drawn according to the distribution $H_{i,j}\in \Delta(\Sigma)$ over the signals. 
In the problem of deterministic feedback studied in the paper so far, the signal $H_{i,j}=\sigma$ was identified with the Dirac distribution $\delta_{\sigma}$. 

Given the matrix $H$ of distributions on $\Sigma$, we can construct, for each row $i$, a matrix $\Xi_i \in \reals^{s_i \times M}$ as 
$$\Xi_i (k,j) \deq H_{i,j}(\sigma_k)$$
where the set $\sigma_1,\ldots,\sigma_{s_i}$ is the union of supports of $H_{i,1},\ldots,H_{i,M}$. Columns of $\Xi_i$ are now distributions over signals. Given the actions $I_t$ and $j_t$ of the player and the opponent, the feedback provided to the player can be equivalently written as $S^t_{I_t} e_{j_t}$ where each column $r$ of the random matrix $S^t_{I_t}\in \reals^{s_i \times M}$ is a standard unit vector drawn independently according to the distribution given by the column $r$ of $\Xi_{i}$. Hence, $\E S^t_i = \Xi_i$.

As before, the matrix $\Xi_{(i,j)}$ is constructed by stacking $\Xi_i$ on top of $\Xi_j$. The local observability condition, adapted to the case of random signals, can now be stated as:
$$\ell_i-\ell_j \in \text{Im}~\Xi^\tr_{(i,j)}$$
for all neighboring actions $i,j$.

Let us specify the few places where the analysis slightly differs from the arguments of the paper. Since we now have an extra (independent) source of randomness, we define $\F_t$ to be the $\sigma$-algebra generated by the random variables $\{k_1,I_1,S^1 \ldots,k_t,I_t,S^t\}$
where $S^t$ is the random matrix obtained by stacking all $S^t_i$. We now define the estimates 
$$ b^r_{(i,j)} \deq v_{i,j}^\tr \left[ \begin{array}{c}
 \ind{I_r = i} S^t_i \\
 \ind{k_r=i}\ind{I_r = j} S^t_j/q^r_i(j)
\end{array} \right] e_{j_r}  ~, ~~~~~~ \forall r \in \{\tau_i(s-1)+1,\ldots,\tau_i(s)\}, ~\forall j\in N_i $$
with the only modification that $S^t_i$ and $S^t_j$ are now random variables. Equation~\eqref{eq:unbiased_b} now reads
\begin{align}
	\E \left[ b^t_{(i,j)} | \F_{t-1} \right] &= \sum_{k=1}^N p^t_k q^t_k(i) \cdot v_{i,j}^\tr \left[ \begin{array}{c}
 \Xi_i  \nonumber \\
 0 
\end{array} \right] e_{j_t} + p^t_i q^t_i(j) \cdot v_{i,j}^\tr \left[ \begin{array}{c}
 0 \\
 \Xi_j/q^t_i(j))
\end{array} \right] e_{j_t} \nonumber  \\
&= p^t_i v_{i,j}^\tr \Xi_{(i,j)} e_{j_t} \nonumber  \\
&= p^t_i (e_j-e_i)^\tr L e_{j_t} \ .
\end{align}
The rest of the analysis follows as in Section~\ref{sec:analysis}, with $\Xi$ in place of $S$.

\section*{Acknowledgements}
We thank Vianney Perchet and Gilles Stoltz for their helpful comments on the first draft of this paper.

\bibliographystyle{plain}
\bibliography{partial_monitoring}
\end{document}